\documentclass[12pt,twoside]{article}
\usepackage{color}
\usepackage{booktabs}
\usepackage{verbatim}
\usepackage{latexsym}
\usepackage{psfrag}
\usepackage[hyperfootnotes=false]{hyperref}
\usepackage{amsfonts}
\usepackage{amsmath}
\usepackage{amssymb}
\usepackage{bbm}
\usepackage{amsthm}
\usepackage{comment}
\usepackage{enumerate}
\usepackage{url}
\usepackage{graphicx}
\usepackage{rotating}
\usepackage{epsfig}
\usepackage[stable]{footmisc}
\usepackage{endnotes}
\usepackage[english]{babel}
\usepackage{makecell}
\usepackage{rotating}
\usepackage{multirow}
\usepackage[hmargin=3.45cm,vmargin=3cm]{geometry}
\usepackage{algorithm} 

\usepackage{algorithmic} 

\theoremstyle{definition}

\theoremstyle{plain}
\newtheorem{theorem}{Theorem}
\newtheorem{proposition}{Proposition}
\newtheorem{corollary}{Corollary}
\newtheorem{lemma}{Lemma}

\theoremstyle{definition}

\theoremstyle{remark}

\newcommand{\mE}{\mathcal{E}}
\newcommand{\barf}{\bar{f}_{\lambda}}

\renewcommand{\hat}[1]{\widehat{#1}}

\newcommand{\mnorm}[1]{\left\vert\kern-1.5pt\left\vert\kern-1.5pt\left\vert #1\right\vert\kern-1.5pt\right\vert\kern-1.5pt\right\vert}
\begin{document}

\begin{center}

{\bf{\Large{On the Feasibility of Distributed Kernel Regression for Big Data}}}

\vspace*{.3in}

\begin{tabular}{cccc}
 Chen Xu$^1$ &  Yongquan Zhang$^2$ & Runze Li$^1$ \\
 \texttt{cux10@psu.edu} &
  \texttt{zyqmath@163.com}  &
   \texttt{rli@stat.psu.edu}
\end{tabular}

\vspace*{.2in}

\begin{tabular}{cc}
  $^1$ The Methodology Center &  $^2$   Department of Mathematics \\
  The Pennsylvania State University &  China Jiliang University  \\
 State College, PA, USA, 16801 & Hangzhou, Zhejiang, China, 310018
\end{tabular}

\vspace*{.2in}
\today


\begin{abstract}
\noindent
In modern scientific research, massive datasets with huge numbers of observations are frequently encountered. To facilitate the computational process, a divide-and-conquer scheme is often used for the analysis of big data. In such a strategy, a full dataset is first split into several manageable segments; the final output is then averaged from the individual outputs of the segments.
Despite its popularity in practice, it remains largely unknown that whether such a distributive strategy provides valid theoretical inferences to the original data. In this paper, we address this fundamental issue for the distributed kernel regression (DKR), where the algorithmic feasibility is measured by the generalization performance of the resulting estimator. To justify DKR, a uniform convergence rate is needed for bounding the generalization error over the individual outputs, which brings new and challenging issues in the big data setup.
Under mild conditions, we show that, with a proper number of segments, DKR leads to an estimator that is generalization consistent to the unknown regression function. The obtained results justify the method of DKR and shed light on the feasibility of using other distributed algorithms for processing big data. The promising preference of the method is supported by both simulation and real data examples.

\end{abstract}
\end{center}

\section{Introduction}
The rapid development in data generation and acquisition
has made a profound impact on knowledge discovery.
Collecting data with unprecedented sizes and complexities is now feasible in many scientific fields.
For example, a satellite takes thousands of high resolution images per day; a Walmart store has millions of transactions per week; and Facebook generates billions of posts per month.
Such examples also occur in agriculture, geology, finance, marketing, bioinformatics, and Internet studies among others.
The appearance of big data brings great opportunities for extracting new information and discovering subtle patterns.
Meanwhile, their huge volume also poses many challenging issues to the traditional data analysis, where a dataset is typically processed on a single machine.
In particular, some severe challenges are from the computational aspect, where the storage bottleneck and algorithmic feasibility need to be faced.
Designing effective and efficient analytic tools for big data has been a recent focus in the statistics and machine learning communities \cite{wu2014data}.

In the literature, several strategies have been proposed for processing big data.
To overcome the storage bottleneck, \textit{Hadoop} system was developed to conduct distributive storage and parallel processing.
The idea of \textit{Hadoop} follows from a natural divide-and-conquer framework, where a large problem is divided into several manageable subproblems and the final output is obtained by combining the corresponding sub-outputs.
With the aid of \textit{Hadoop}, many machine learning methods can be re-built to their distributed versions for the big data analysis. 
For examples, McDonald et al. \cite{mcdonald2010distributed} considered a distributed
training approach for structured perception, while Kleiner et al. \cite{kleiner2012big} introduced a distributed bootstrap method.
Recently, similar ideas have also been applied to statistical point estimation \cite{li2013statistical}, kernel ridge regression \cite{zhang2013divide}, matrix factorization \cite{mackey2011divide}, and principal component analysis \cite{zhao2012recursive}.

To better understand the divide-and-conquer strategy, let us consider an illustrative example as follows.
Suppose that a dataset consists of $N=1,000,000$ random samples $\{(x_i, y_i)\}_{i=1}^N \subset\mathbb{R}^{d}\times\mathbb{R}$ with dimension $d=100$.
We assume that the data follow from a linear model $y_{i} = x^{T}_{i}\beta + \varepsilon$ with a random noise $\varepsilon$. The goal of learning is to estimate the regression coefficient $\beta$. Let $Y=(y_{1}, \ldots, y_{N})^{T}$ be the $N$-dimensional response vector and $X=(x_{1}, \ldots, x_{N})^{T}$ be the $N \times d$ covariate matrix.
Apparently, the huge sample size of this problem makes the single-machine-based least squares estimate $\hat{\beta}=(X^{\top}X)^{-1}X^\top Y$ computationally costly.
Instead, one may first evenly distribute the $N$ samples into $m$ local machines and obtain $m$ sub-estimates $\hat{\beta}_{j}$ based on $m$ independent running.
The final estimate of $\beta$ can then be obtained by averaging the $m$ sub-estimates $\bar{\beta}=  \sum_{j=1}^{m} \hat{\beta}_{j}/m$.
Compared with the traditional method, such a distributive learning framework utilizes the computing power of multiple machines, which avoids the direct storage and operation on the original full dataset. We further illustrate this framework in Figure 1 and refer to it as a distributed algorithm.

\begin{figure}
\begin{center}
  \includegraphics[width=350pt]{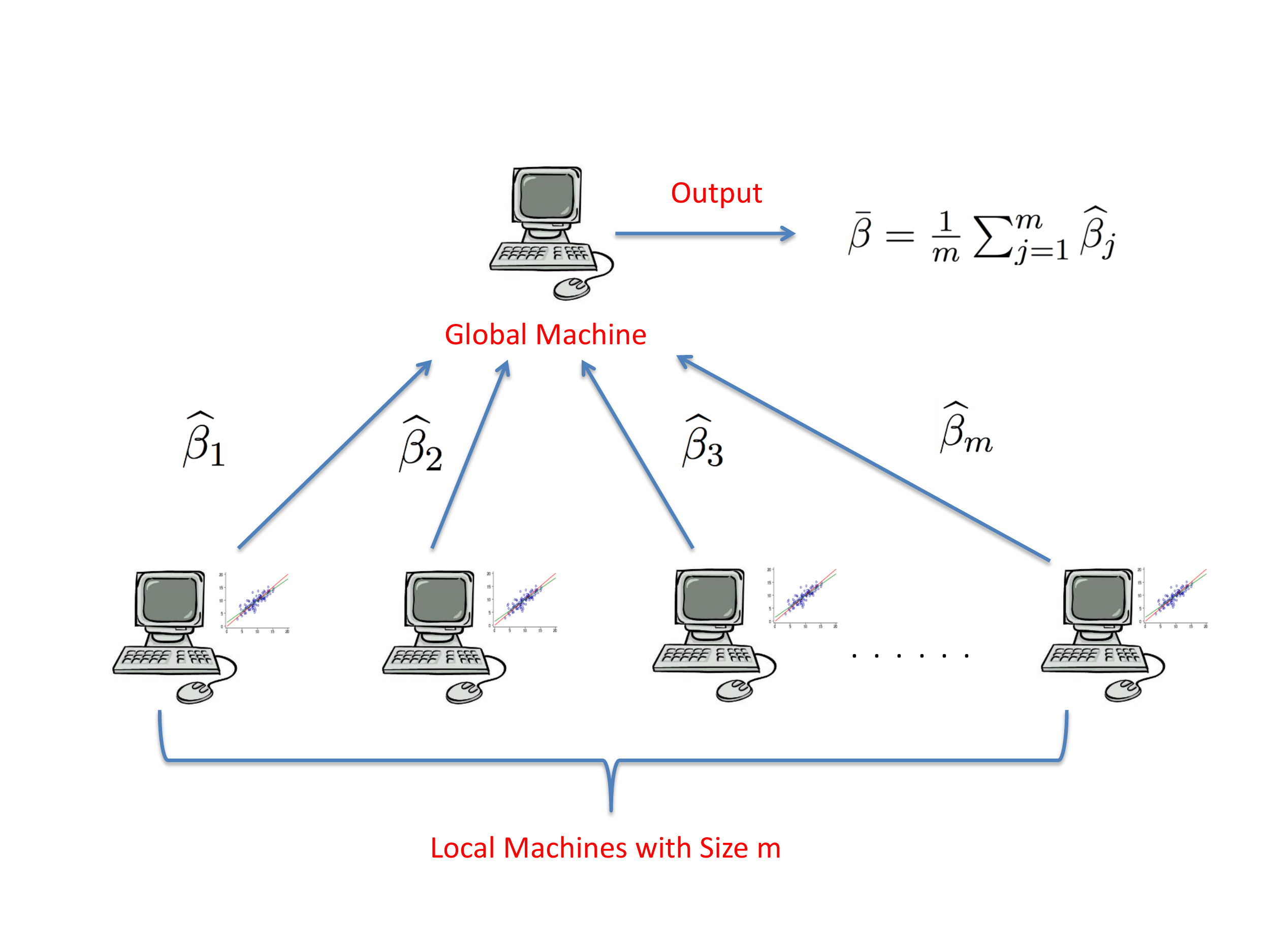}
  \end{center}
 \caption{A divide-and-conquer learning framework.}\label{figureDSRA}
\end{figure}

\begin{figure}
\begin{center}
  \includegraphics[width=300pt]{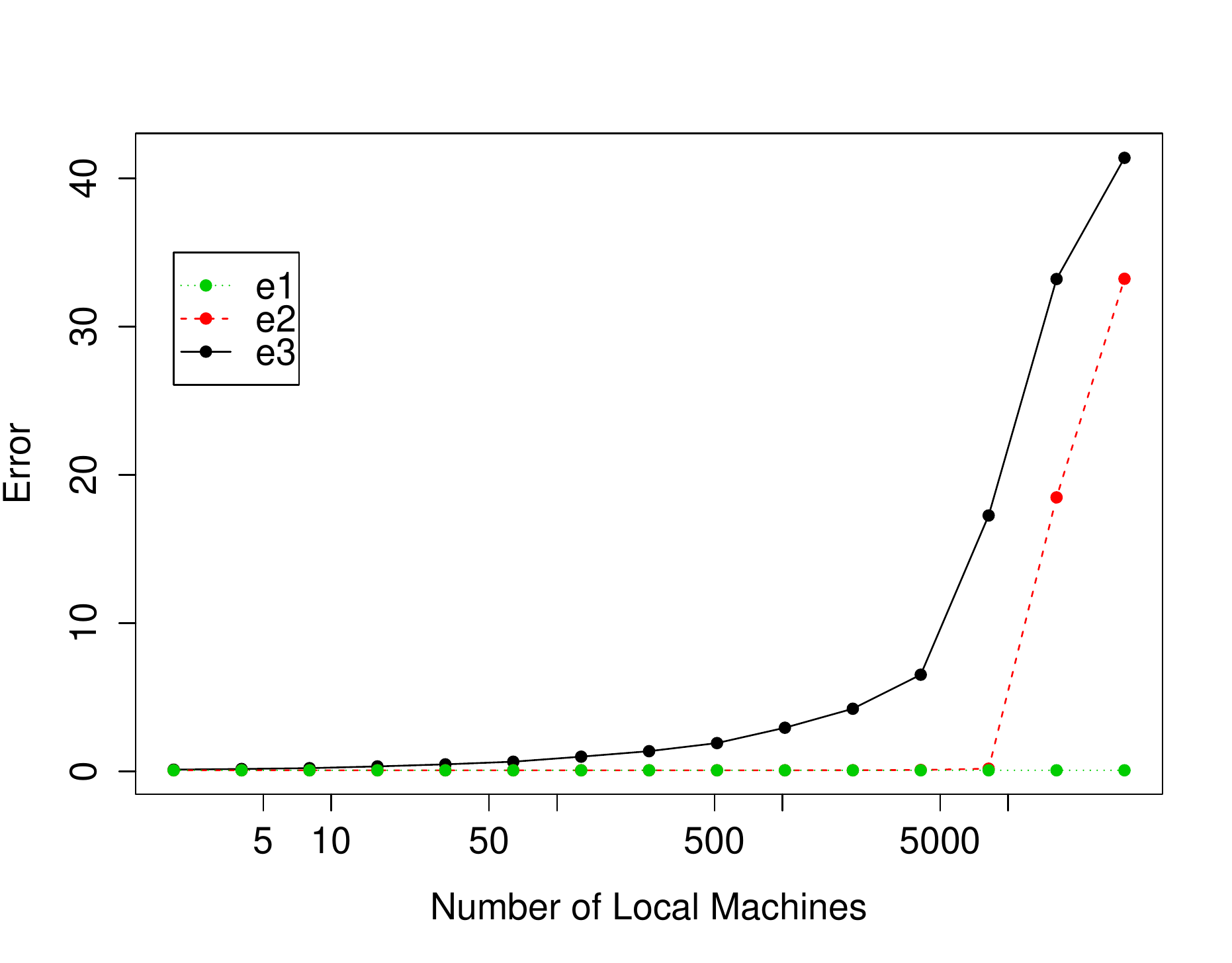}
  \end{center}
 \caption{Estimation errors for the distributed regression.}\label{figureDistributed_regression}
\end{figure}

The distributed algorithm provides a computationally viable route for learning with big data.
However, it remains largely unknown that whether such a divide-and-conquer scheme indeed provides valid theoretical inferences to the original data.
For point estimation, Li et al. \cite{li2013statistical} showed that the distributed moment estimation is consistent, if an unbiased estimate is obtained for each of the sub-problems.
For kernel ridge regression, Zhang et al. \cite{zhang2013divide} showed that, with appropriate tuning parameters, the distributed algorithm does lead to a valid estimation. To provide some insights on the feasibility issue, we numerically compare the estimation accuracy of $\bar{\beta}$ with that of $\hat{\beta}$ in the aforementioned example.
Specifically, we generate $x_{i}$ independently from $N(0, I_{d \times d})$ and set $\beta$ based on $d$ independent observations from $U[0,1]$.
The value of $y_{i}$ is generated from the presumed linear model with $\varepsilon \sim N(0,1)$. We then randomly distribute the full data to $m \in [2^{0}, 2^{15}]$ local machines and output $\bar{\beta}$ based on $m$ local ridge estimates $\hat{\beta}_{j}$ for $j=1, \ldots, m$.
In Figure 2, we plot the estimation errors versus the number of local machines $m$ based on three types of estimators: $e_1=\|\beta-\hat{\beta}\|_2^2$, $e_2=\|\beta-\bar{\beta}\|_2^2$, and $e_3=\min_{j}\|\beta-\hat{\beta}_{j}\|_2^2$.
For a wide range of $m$, it seems that the distributed estimator $\bar{\beta}$ leads to a similar accuracy as the traditional $\hat{\beta}$ does.
However, this argument tends to be false when $m$ is overly large. This observation brings an interesting but fundamental question for using the distributed algorithm in regression: under what conditions the distributed estimator provides an effective estimation of the target function?
In this paper, we aim to find an answer to this question and provide more general theoretical support for the distributed regression.

Under the kernel-based regression setup, we propose to take the generalization consistency as a criterion for
measuring the feasibility of the distributed algorithms.
That is, we regard an algorithm is theoretically feasible if its
generalization error tends to zero as the number of observations $N$ goes to infinity.
To justify the distributed regression, a uniform convergence rate is needed for bounding the generalization error over the $m$ sub-estimators.
This brings new and challenging issues in analysis for the big data setup.
Under mild conditions, we show that the distributed kernel regression (DKR) is feasible when the number of its distributed sub-problems is moderate.
Our result is applicable to many commonly used regression models, which incorporate a variety of loss, kernel, and penalty functions.
Moreover, the feasibility of DKR does not rely on any parametric assumption on the true model. It therefore provides a basic and general understanding for the distributed regression analysis.
We demonstrate the promising performance of DKR via both simulation and real data examples.

The rest of the paper is organized as follows. In Section \ref{sectionDSRA}, we introduce model setup and formulate the DKR algorithm.
In Section \ref{sectionmain}, we establish the generalization consistency and justify the feasibility of DKR.
In Section \ref{sectoinapplication}, we show numerical examples to support the good performance of DKR.
Finally, we conclude the paper in Section \ref{sectionConclusion} with some useful remarks.

\section{Distributed Kernel Regression}\label{sectionDSRA}

\subsection{Notations}

Let $Y \in [-M,M] \subset  \mathbb{R}$ be a response variable bounded by some $M>0$ and $X \in \mathcal{X} \subset \mathbb{R}^{d}$ be its $d$-dimensional covariate drawn from a compact set $ \mathcal{X}$. Suppose that $Z= X \times Y$ follows from a fixed but unknown distribution $\rho$ with its support fully filled on $\mathcal{Z}=[-M, M] \times \mathcal{X}$.
Let $S=\{z_{i}=(y_{i}, x_{i}), i=1, \ldots, N\}$ be $N$ independent observations collected from $Z$. The goal of study is to estimate the potential relationship $f^{*}: \mathcal{X} \rightarrow R$ between $X$ and $Y$ through analyzing $S$.

Let $\ell(.)$ be a nonnegative loss function and $f$ be an arbitrary mapping from $\mathcal{X}$ to $\mathbb{R}$.
We use
\begin{equation*}
\mathcal{E}(f)=\mathbb{E}_{z}[{\ell }(f,z)]=\int_{\mathcal{Z}}{{\ell }(f,z)} d\rho
\end{equation*}
to denote the expected risk of $f$. The minimizer $f_{\rho}= \arg \min \mathcal{E}(f)$ is called the regression function, which is an oracle estimate under $\ell$ and thus serves as a benchmark for other estimators. Since $\rho$ is unknown, $f_{\rho}$ is only conceptual. Practically, it is common to estimate $f^{*}$ through minimizing a regularized empirical risk
\begin{equation} \label{RERM}
\min\limits_{f\in \mathcal{F}}\Big\{ \mathcal{E}_{S}(f) + \lambda \|f\| \Big\},
\end{equation}
where $\mathcal{F}$ is a user-specified hypothesis space, $\mathcal{E}_{S}(f)= \sum_{i=1}^{N}{\ell }(f,z_{i})/N$ is the empirical risk,  $\|.\|$ is a
norm in $\mathcal{F}$, and $\lambda \geq 0$ is a regularization parameter.

Framework (\ref{RERM}) covers a broad range of regression methods. In the machine learning community, it is popular to set $\mathcal{F}$ by a reproducing kernel Hilbert space (RKHS).
Specifically, let $K:\mathcal{X}\times \mathcal{X}\rightarrow \mathbb{R}$ be a
continuous, symmetric, and semi-positive definite kernel function.
The RKHS $\mathcal{H}_{K}=\overline{\mbox{span}}\{K(x,.), x \in \mathcal{X}\}$ is a Hilbert space
of $L^{2}$-integrable functions induced by $K$. For any $f=\sum_{i}\alpha_{i}K(u_{i}, .)$ and $g=\sum_{i}\beta_{i}K(v_{i}, .)$, their inner product is defined by
\[
<f,g>_{K} = \sum_{i, j}\alpha_{i} \beta_{j}K(u_{i}, v_{j})
\]
and the kernel $L_{2}$ norm is given by $\|f\|^{2}_{K} = <f,f>_{K}$.
It is easy to verify that
\begin{equation} \label{RD}
f(x) = < f, K(x,\cdot )> _{\mathcal{H}_{K}}
\end{equation}
for any $f\in\mathcal{H}_{K}$. Therefore, $K$ is a reproducing kernel of $\mathcal{H}_{K}$.
Readers may refer to \cite{Berlinet2004} \cite{wahba1990} for more detailed discussions about RKHS.

Let $\mathcal{C}(\mathcal{X})$ denote the space of continuous functions on $\mathcal{X}$. It is known that $\mathcal{H}_{K}$ is dense in $\mathcal{C}(\mathcal{X})$ with appropriate choices of $K$ \cite{Micchelli-et-al(2006)}. This property makes $\mathcal{H}_{K}$ a highly flexible space to estimate an arbitrary $f^* \in \mathcal{C}(\mathcal{X})$. In this paper, we follow framework (\ref{RERM}) with $\mathcal{F}=\mathcal{H}_{K}$ and $\|.\| = \|.\|_{K}^{p}$ for some $p>0$.

\subsection{The DKR Algorithm}

We now consider (\ref{RERM}) in the big data setup.
In particular, we assume that sample $S$ is too big to be processed in a
single machine and thus we need to use its distributed version.
Suppose $S$ is evenly and randomly assigned to $m$ local
machines, with each machine processing $n=N/m$ samples. We denote by $%
S_{j},$ $j=1,2,\dots ,m$ the sample segment assigned to the $j$th machine.
The global estimator is then
constructed through taking average of the $m$ local
estimators.
Specifically, by setting $\mathcal{F}=\mathcal{H}_{K}$ in (\ref{RERM}),
this strategy leads to the distributed kernel regression (DKR), which is described as Algorithm \ref{DSRA}.

\begin{algorithm}[htb]

\renewcommand{\algorithmicrequire}{\textbf{Input:}}

\renewcommand\algorithmicensure {\textbf{Output:} }

\caption{The DKR Algorithm}

\label{DSRA}
\begin{algorithmic}[1]

\REQUIRE ~
$S$,  $K$,  $\lambda$, $m$ \\

\ENSURE ~
 $\bar{f}$\\
~\\
\STATE Randomly split $S$ into $m$ sub-samples $S_{1}$, \ldots, $S_{m}$ and store them separately on $m$ local machines.\\ \vspace{0.2cm}

\STATE Let $T_{M}[.]$ be a truncation operator with a cutoff threshold $M$. For $j= 1, 2, \dots, m$, find a local estimator based on $S_{j}$ by
\[
 \hat{f}_{j} = T_{M}\left[ f_{j} \right],
\]
where
\begin{equation*}
f_{j} = \arg\min\limits_{f\in\mathcal{H}_K}\Big\{\frac{1}{n}\sum_{z_i\in S_j}\ell(f,z_i)+\lambda\|f\|_{K}^p\Big\}.
\end{equation*}
\\

\STATE Combine $\hat{f}_{j}$s to get a global estimator
\begin{eqnarray*}
\bar{f}= \frac{1}{m}\sum_{j=1}^m\hat{f}_{j}.
\end{eqnarray*}
\end{algorithmic}

\end{algorithm}

By representer theorem \cite{scho2001}, $f_{j}$ in step 2 of DKR can be constructed from $\mbox{span}\{K(x_{i},.), x_i \in S_{j} \}$.
This allows DKR to be practically carried out within finite $n$-dimensional subspaces. The distributive framework of DKR enables parallel processing and thus is appealing to the analysis of big data.
With $m=1$, DKR reduces to the regular kernel-based learning, which has received a great deal of attention in the literature  \cite{scholkopf2001learning} \cite{wu2006learning} \cite{zhang2005learning}. With quadratic $\ell$ and $p=2$, Zhang et. al.~\cite{zhang2013divide} conducted a feasibility analysis for DKR with $m>1$. Unfortunately, their results are built upon the close-form solution of $f_{j}$ and thus are not applicable to other DKR cases. In this work, we attempt to provide a more general feasibility result for using DKR in dig data.

\section{Consistency of DKR}\label{sectionmain}

\subsection{Preliminaries and Assumptions}

In regression analysis, a good estimator of $f^*$ is expected not only to fit training set $S$ but also to predict the future samples from $Z$.
In the machine learning community, such an ability is often referred to as the generalization capability.
Recall that $f_{\rho}$ is a conceptual oracle estimator, which enjoys the lowest generalization risk in a given loss.
The goodness of $\bar{f}$ can be typically measured by
\begin{eqnarray}\label{Gen_error}
\mE(\bar{f})-\mE(f_{\rho})= \mathbb{E}_{z}[\ell(\bar{f}, z) - \ell(f_{\rho}, z)].
\end{eqnarray}
A feasible (consistent) $\bar{f}$ is then required to have
generalization error (\ref{Gen_error}) converge to zero as $N \rightarrow \infty$.
When the quadratic loss is used, the convergence of (\ref{Gen_error}) also leads to the convergence of $\|\bar{f} - f_{\rho}\|_{2}$,
which responds to the traditional notion of consistency in statistics.

When $\ell$ is convex, Jensen's inequality implies that
\[
\mE(\bar{f})-\mE(f_{\rho}) \leq \frac{1}{m}\sum_{j=1}^{m}[\mE(\hat{f}_{j})-\mE(f_{\rho}) ].
\]
Therefore, the consistency of $\bar{f}$ is implied by the uniform consistency of the $m$ local estimators $\hat{f}_{j}$ for $j=1, \ldots, m$.
Under appropriate conditions, this result may be straightforward in the fixed $m$ setup. However, for analyzing big data, it is particularly desired to have $m$ associated with sample size $N$.
This is because the number of machines needed in an analysis is usually determined by the scale of that problem. The larger a dataset is, the more machines are needed.
This in turn suggests that, in asymptotic analysis, $m$ may diverge to infinity as $N$ increases. This liberal requirement of $m$ poses new and challenging issues to justify $\bar{f}$ under the big data setup.

Clearly, the effectiveness of a learning method relies on the
prior assumptions on $f_{\rho}$ as well as the choice of $\ell$. For the convenience of
discussion, we assess the performance of DKR under the following conditions.

\begin{enumerate}
 \item [A1] $f_\rho\in \mathcal{C}(\mathcal{X})$ and $\|f_\rho\|_{\infty} \leq M$, where $\|.\|_{\infty}$ denotes the function supremum norm.

 \item [A2] The loss function $\ell$ is convex and nonnegative. For any $f_{1},f_{2}\in \mathcal{C}(\mathcal{X})$ and $z \in \mathcal{Z}$,
there exists a constant $L$ such that
\[
|\ell (f_{1},z)-\ell (f_{2},z)|\leq L\Vert f_{1}-f_{2}\Vert _{\infty }.
\]

 \item [A3] For any $\omega>0$ and $g \in \mathcal{C}(\mathcal{X})$, there exists a $f \in \mathcal{H}_{K}$, such that $\|f-g\|_{\infty} < \omega$.
Moreover, let $\mathcal{B}_{R}=\{f \in \mathcal{H}_{K}, \|f\|_{\infty} \leq R \}$ for some $R >0$. There exists constants $C_0$, $s>0$, such that
\[
\log \mathcal{N}_{\infty}(\mathcal{B}_{1}, \gamma) \leq C_{0} \gamma^{- s},
\]
where $\mathcal{N}_{\infty}(\mathcal{F}, \gamma)$ denotes the covering number of a set $\mathcal{F}$ by balls of radius $\gamma$ with respect to $\|.\|_{\infty}$.

\end{enumerate}


Condition A1 is a regularity assumption on $f_{\rho}$, which can be trivial in applications.
For the quadratic loss, we have $f_{\rho}(X)=\mathbb{E}(Y|X)$ and thus A1 holds naturally with $Y \in[-M, M]$.
Condition A2 requires that $\ell(f, z)$ is Lipschitz continuous in $f$. It is satisfied by many commonly used loss functions for regression analysis.
Condition A3 corresponds to the notion of universal kernel in \cite{Micchelli-et-al(2006)}, which implies that $\mathcal{H}_{K}$ is dense in $\mathcal C({\mathcal{X}})$.
It therefore serves as a prerequisite for estimating an arbitrary $f^{*} \in \mathcal{C}(\mathcal{X})$ from $\mathcal H_{K}$.
A3 also requires that the unit subspace of $\mathcal{H}_K$ has a polynomial complexity. Under our setup, a broad choices of $K$ satisfy this
condition, which include the popular Gaussian kernel as a special case \cite{zhou2002covering} \cite{zhou2003capacity}.

\subsection{Generalization Analysis}

To justify DKR, we decompose (\ref{Gen_error}) by
\begin{eqnarray}\label{sampleerror}
\mE(\bar{f})-\mE(f_{\rho})&=&\underbrace{\mE_S(f)-\mE(f)+\mE(\bar{f})-\mE_S(\bar{f})}_{\text{sample error}}\\ \label{hypothesiserror}
&+& \underbrace{\mE_S(\bar{f})-\mE_S(f)}_{\text{hypothesis error}}\\
\label{approximationerror}
&+&\underbrace{\mE(f)-\mE(f_\rho)}_{\text{approximation error}},
\end{eqnarray}
where $f$ is an arbitrary element of $\mathcal{H}_{\mathcal{K}}$.
The consistency of $\bar{f}$ is implied if (\ref{Gen_error}) has convergent sub-errors in (\ref{sampleerror})-(\ref{approximationerror}).
Since $f \in \mathcal{H}_{\mathcal{K}}$ is arbitrary, (\ref{approximationerror}) measures how close the oracle $f_{\rho}$ can be approximated from the candidate space $\mathcal{H}_{\mathcal{K}}$.
This is a term that purely reflects the prior assumptions on a learning problem.
Under Conditions A1-A3, with a $f$ such that $\|f - f_{\rho}\|\leq N^{-1}$, (\ref{approximationerror}) is naturally bounded by $L/N$.
We therefore carry on our justification by bounding the sample and hypothesis errors.

\subsubsection{Sample Error Bound}
Let us first work on the sample error (\ref{sampleerror}), which describes the difference between the expected loss and the empirical loss for an estimator.
For the convenience of analysis, let us rewrite (\ref{sampleerror}) as
\begin{eqnarray} \nonumber
&&\mE_S(f)-\mE(f)+\mE(\bar{f})-\mE_S(\bar{f}) \\ \label{decomposesampleerror}
&&  = \Big\{\frac{1}{N}\sum_{i=1}^{N}\xi_1(z_i)-\mathbb{E}_{z}(\xi_1)\Big\} + \Big\{\mathbb{E}_{z}(\xi_2)-\frac{1}{N}\sum_{i=1}^{N}\xi_2(z_i)\Big\},
\end{eqnarray}
where $\xi_1(z) = \ell(f,z)-\ell(f_\rho,z)$ and $\xi_2(z) =\ell(\bar{f},z)-\ell(f_\rho,z)$.
It should be noted that the randomness of $\xi_1$ is purely from $Z$, which makes $\mathbb{E}_{z}(\xi_1)$ a fixed quantity and $\sum_{i=1}^{N}\xi_1(z_{i})/N$ a sample mean of independent observations. For $\xi_2$, since $\bar{f}$ is an output of $S$, $\mathbb{E}_{z}(\xi_2)$ is random in $S$ and $\xi_2(z_{i})$s are dependent with each other.
We derive a probability bound for the sample error through investigating (\ref{decomposesampleerror}).

To facilitate our proofs, we first state one-side Bernstein inequality as the following lemma.
\begin{lemma}\label{bernstein}
Let $y_{1}, \ldots, y_{N}$ be $N$ independently and identically distributed random variables with $\mathbb{E}(y_{1})= \mu$ and $\mbox{var}({y_{1}})= \sigma^{2}$.
If $|y_{1} - \mu| \leq T$ for some $T > 0$, then for any $\varepsilon>0$,
\[
\mathbb{P} \left\{\frac{1}{N}\sum_{i=1}^{N}y_{i} -\mu\geq\varepsilon\right\} \leq
\exp \left\{\frac{-N\varepsilon^2}{2(\sigma^2
+\varepsilon T/3)}\right\}.
\]
\end{lemma}
The probability bounds for the two terms of (\ref{decomposesampleerror}) are given respectively in the following propositions.

\begin{proposition}\label{firstbound}
Suppose that Conditions A1-A2 are satisfied. For any $0<\delta<1$ and $f \in \mathcal{H}_{K}$, we have
\begin{equation*}
\mathbb{P}\left\{\frac{1}{N}\sum_{i=1}^{N}\xi_1(z_i)-\mathbb{E}_{z}(\xi_1)  \leq 2L \|f - f_{\rho}\|_{\infty} \left(\frac{\log(1/\delta)}{N}+\sqrt{\frac{\log(1/\delta)}{N}}\right) \right\}\geq 1-\delta.
\end{equation*}
\end{proposition}

\begin{proof}

Let $f$ be an arbitrary function in $\mathcal{H}_{K}$. By Condition A2, we have
 $$|\xi_1(z)|=|\ell(f,z)-\ell(f_\rho,z)|\leq L\|f-f_\rho\|_{\infty}$$
for some constant $L>0$.
This implies that $\mbox{var}(\xi_1)\leq L^{2}\|f-f_\rho\|_{\infty}^{2}$ and $|\xi_1-\mathbb{E}_{z}(\xi_1)|\leq 2L\|f-f_\rho\|_{\infty}$.
By Lemma \ref{bernstein}, we have,
 \begin{equation}\label{prop1-1}
 \mathbb{P}\Big\{\frac{1}{N}\sum_{i=1}^{N}\xi_1(z_i)-\mathbb{E}_{z}(\xi_1)\geq \varepsilon \Big\}\leq \exp \left\{-\frac{Nt^2}{2(L^2\|f-f_\rho\|_{\infty}^2+2/3L\|f-f_\rho\|_{\infty}t)}\right\}
\end{equation}
 for any $\varepsilon >0$. Denoting the right hand side of (\ref{prop1-1}) by $\delta$, we have
\begin{equation}\label{prop1-2}
N \varepsilon^2+\frac{4}{3}L\|f-f_\rho\|_{\infty}\log\delta \varepsilon +2L^2\|f-f_\rho\|_{\infty}^2\log\delta=0.
\end{equation}
The positive root of (\ref{prop1-2}) is given by
\begin{eqnarray} \nonumber
 \varepsilon^* &=& \frac{\frac{4}{3}L \|f-f_\rho\|_{\infty} \log 1/\delta+L \|f-f_\rho\|_{\infty} \sqrt{\frac{16}{9}\log^21/\delta+8N\log 1/\delta}}{2N} \\ \nonumber
 &\leq& L \|f-f_\rho\|_{\infty} \left(\frac{4\log 1/\delta}{3N}+\sqrt{\frac{2\log 1/\delta}{N}}\right) \\ \label{prop1-3}
 &\leq& 2L \|f - f_{\rho}\|_{\infty} \left(\frac{\log(1/\delta)}{N}+\sqrt{\frac{\log(1/\delta)}{N}}\right).
\end{eqnarray}
The proposition is proved by setting $\varepsilon = \varepsilon^*$ in (\ref{prop1-1}).
\end{proof}

\begin{proposition}\label{secondbound}
Suppose that Conditions A1-A3 are satisfied. For any $0<\delta<1$ and $f \in \mathcal{H}_{K}$, we have
\begin{equation*}
\mathbb{P}\left\{ \mathbb{E}_{z}(\xi_2)-\frac{1}{N}\sum_{i=1}^{N}\xi_2(z_i) \leq  12ML\left(\frac{V(N, \delta) +  \sqrt{V(N, \delta)N}}{N} \right) + N^{- 1/(s+2)} \right\} \geq 1 - \delta
\end{equation*}
where $V(N, \delta) = C_{0}(8LM N^{1/(s+2)})^{s} - \log \delta$.
\end{proposition}

\begin{proof}
Let $\mathcal{D}_{M}=\{f \in \mathcal{C}(\mathcal{X}), \|f\|_{\infty} \leq M \}$. Under Condition A3, $\mathcal{B}_{2M} \subset \mathcal{H}_{K}$ is dense in $\mathcal{D}_{M}$. Therefore,
for any $\epsilon > 0$, there exists a $g_{\epsilon} \in \mathcal{B}_{2M}$, such that $\| \bar{f} - g_{\epsilon}\|_{\infty} < \epsilon$. By A2, we further have
\[
\ell(\bar{f}, z) - \ell(g_{\epsilon}, z) \leq L\epsilon.
\]
Consequently,
\begin{eqnarray} \nonumber
\mathbb{E}_{z}(\xi_2)-\frac{1}{N}\sum_{i=1}^{N}\xi_2(z_i) &=& \mE(\bar{f}) - \mE(f_{\rho}) - [\mE_S(\bar{f}) - \mE_S(f_{\rho})] \\ \label{Prop2-1}
&\leq&  \mE(g_{\epsilon}) - \mE(f_{\rho}) - [\mE_S(g_{\epsilon}) - \mE_S(f_{\rho})] + 2L\epsilon.
\end{eqnarray}
Let $U_{\gamma} \subset \mathcal{B}_{2M}$ be a cover of $\mathcal{B}_{2M}$ by balls of radius $\gamma$ with respect to $\|.\|_\infty$. With $\epsilon \rightarrow 0$, (\ref{Prop2-1}) implies that
\begin{eqnarray} \nonumber
&&\mathbb{P}\left\{ \mathbb{E}_{z}(\xi_2)-\frac{1}{N}\sum_{i=1}^{N}\xi_2(z_i) \geq \varepsilon \right \} \\ \nonumber
&\leq& \mathbb{P}\left\{\sup_{g \in \mathcal{B}_{2M}} \mE(g) - \mE(f_{\rho}) - [\mE_S(g) - \mE_S(f_{\rho})]  \geq \varepsilon   \right \}  \\  \nonumber
&\leq& \mathbb{P}\left\{\sup_{g \in U_{\gamma}} \mE(g) - \mE(f_{\rho}) - [\mE_S(g) - \mE_S(f_{\rho})] \geq \varepsilon  - 2L\gamma  \right \} \\ \nonumber
&\leq& \mathcal{N}_{\infty}(\mathcal{B}_{2M}, \gamma) \max_{g \in U_{\gamma}} \mathbb{P}\left\{\mE(g) - \mE(f_{\rho}) - [\mE_S(g) - \mE_S(f_{\rho})] \geq \varepsilon  - 2L\gamma  \right \} \\ \label{Prop2-2}
&\leq& \mathcal{N}_{\infty}(\mathcal{B}_{2M}, \gamma) \exp\left\{- \frac{N(\varepsilon  - 2L\gamma )^{2}}{2[9L^{2}M^{2} + 2(\varepsilon  - 2L\gamma)LM]} \right\},
\end{eqnarray}
where the last inequality follows from Lemma \ref{bernstein}. By A3, we have
\begin{eqnarray} \label{Prop2-3}
\mathcal{N}_{\infty}(\mathcal{B}_{2M}, \gamma) = \mathcal{N}_{\infty}(\mathcal{B}_{1}, \gamma/2M) \leq \exp\{C_{0}(2M/\gamma)^{s}  \}.
\end{eqnarray}
Let $\gamma = \varepsilon/4L$. Inequality (\ref{Prop2-2}) together with (\ref{Prop2-3}) further implies that
\begin{eqnarray} \label{Prop2-4}
\mathbb{P}\left\{ \mathbb{E}_{z}(\xi_2)-\frac{1}{N}\sum_{i=1}^{N}\xi_2(z_i) \geq \varepsilon \right \} \leq \exp\left\{C_{0}(\frac{8LM}{\varepsilon})^{s} - \frac{N(\varepsilon )^{2}}{72L^{2}M^{2} + 8\varepsilon LM}  \right\}.
\end{eqnarray}

When $\varepsilon \geq N^{-\tau}$ for some $\tau>0$, (\ref{Prop2-4}) implies that
\begin{eqnarray} \label{Prop2-5}
\mathbb{P}\left\{ \mathbb{E}_{z}(\xi_2)-\frac{1}{N}\sum_{i=1}^{N}\xi_2(z_i) \geq \varepsilon \right \} \leq \exp\left\{C_{0}(8LMN^{\tau})^{s} - \frac{N(\varepsilon )^{2}}{72L^{2}M^{2} + 8\varepsilon LM}  \right\}.
\end{eqnarray}
Denote the right hand side of (\ref{Prop2-5}) by $\delta$. Following the similar arguments in (\ref{prop1-2}) - (\ref{prop1-3}), we have
\begin{eqnarray} \label{Prop2-6}
\mathbb{P}\left\{ \mathbb{E}_{z}(\xi_2)-\frac{1}{N}\sum_{i=1}^{N}\xi_2(z_i) \geq ML\left(\frac{8V(N, \delta) + 6 \sqrt{2V(N, \delta)N}}{N} \right) + N^{-\tau} \right\} \leq \delta,
\end{eqnarray}
where $V(N, \delta)= C_{0}(8LM N^{\tau})^{s} - \log \delta$. The proposition is proved by setting $\tau= 1/(s+2)$, which minimizes the bound order in (\ref{Prop2-6}).
\end{proof}

Based on Propositions \ref{firstbound} and \ref{secondbound}, decomposition (\ref{decomposesampleerror}) implies directly the following probability bound of the sample error.

\begin{theorem}(Sample Error)\label{upperboundsampelerror}
Suppose that Conditions A1-A3 are salified. Let $M' = \max\{2M, \|f -f_{\rho}\|_{\infty}\}$. For any $f \in \mathcal{H}_{K}$ and $0<\delta<1$, we have, with probability at least $1 -\delta$,
\begin{equation} \label{Thm1}
\mE_S(f)-\mE(f) + \mE(\barf) - \mE_S(\barf)\leq 6M'L\left\{\frac{T_1(N, \delta)}{N} + \frac{T_2(N, \delta)}{N^{\frac{1}{2}}} \right\} + \frac{1}{N^{\frac{1}{2+s}}},
\end{equation}
where
\begin{eqnarray*}
T_1(N, \delta) &=&  V(N, \delta/2) +\log(2/\delta), \\
T_2(N, \delta) &=&  \sqrt{V(N , \delta/2)} + \sqrt{\log(2/\delta)}.
\end{eqnarray*}
\end{theorem}
When $\|f- f_{\rho}\|_{\infty}$ is bounded, the leading factor in (\ref{Thm1}) is $\sqrt{V(N, \delta/2)/N}$.
In that case, Theorem \ref{upperboundsampelerror} implies that the sample error (\ref{sampleerror}) has an $O(N^{-1/(2+s)})$ bound in probability.
Under our model setup, this result is general for a broad range of continuous estimators that is bounded above.


\subsubsection{Hypothesis Error Bound}

We now continue our feasibility analysis on the hypothesis error (\ref{hypothesiserror}), which measures the empirical risk difference between $\bar{f}$ and an arbitrary $f$.  When DKR is conducted with $m=1$, $\bar{f}$ corresponds to the single-machine-based kernel learning. By setting $\lambda=0$, the hypothesis error has a natural zero bound by definition.
However, this property is no longer valid for a general DKR with $m>1$.

When $\ell$ is convex, we have (\ref{hypothesiserror}) bounded by
\begin{eqnarray} \nonumber
\mE_{S}(\bar{f}) - \mE_{S}(f) &=& \frac{1}{N}\sum_{i=1}^{N}\ell\left(\frac{1}{m}\sum_{j=1}^{m}\hat{f}_{j},z_{i}\right)
-\frac{1}{N}\sum_{i=1}^{N}\ell\left(f,z_{i}\right) \\ \label{Hyp}
&\leq&  \frac{1}{m}\sum_{j=1}^{m}\left\{\mE_{S}(\hat{f}_{j}) - \mE_{S}(f) \right\}.
\end{eqnarray}
This implies that the hypothesis error of $\bar{f}$ is bounded by a uniform bound of the hypothesis errors over the $m$ sub-estimators.
We formulate this idea as the following theorem.

\begin{theorem}(Hypothesis Error)\label{HypothesisError_Theory}
Suppose that Conditions A1-A3 are satisfied. For any $0<\delta<1$ and $f \in \mathcal{H}_{\mathcal{K}}$, we have, with probability at least $1 - \delta$,
\begin{eqnarray*}
\mE_{S}(\bar{f}) - \mE_{S}(f) \leq 6LM' \left( \frac{T_{1}(n, \delta/2)}{n} + \frac{T_{2}(n, \delta/2) }{n^{\frac{1}{2}}} \right) + \frac{1}{n^\frac{1}{2+s}} +  2\lambda\|f\|_{K}^p,
\end{eqnarray*}
where $M'$, $T_{1}$, and $T_{2}$ are defined in Theorem 1.
\end{theorem}

\begin{proof}
Without loss of generality, we prove the theorem for $\bar{f}$ with $m>1$. Recall that DKR spilt $S$ into $m$ segments $S_{1}, \ldots, S_{m}$. Let $S/S_{j}$ be the sample set with $S_{j}$ removed from $S$ and $\mE_{Q}=\sum_{z_{i} \in Q} \ell(f, z_{i})/q$ be the empirical risk for a sample set $Q$ of size $q$. Under A2, we have $\ell$ is convex and thus
\begin{eqnarray} \nonumber
\mE_{S}(\bar{f}) - \mE_{S}(f) &\leq&  \frac{1}{m}\sum_{j=1}^{m}\left\{\mE_{S}(\hat{f}_{j}) - \mE_{S}(f) \right\} \\ \nonumber
&=& \frac{1}{m}\sum_{j=1}^{m}\left[\frac{m}{N}(\mathcal{E}_{S_{j}}(\hat{f}_{j}) -\mathcal{E}_{S_{j}}(f))+\frac{N-m}{N}(\mathcal{E}_{S/S_{j}}(\hat{f}_{j})-\mathcal{E}_{S/S_{j}}(f))\right] \\  \label{Thm2-1}
&=& \frac{1}{m}\sum_{j=1}^{m}\left[\frac{m}{N}B_{j} + \frac{N-m}{N} U_{j} \right],
\end{eqnarray}
where $B_{j} =(\mathcal{E}_{S_{j}}(\hat{f}_{j}) -\mathcal{E}_{S_{j}}(f))$ and $U_{j}=(\mathcal{E}_{S/S_{j}}(\hat{f}_{j})-\mathcal{E}_{S/S_{j}}(f))$.

Let us first work on the first term of (\ref{Thm2-1}). By definition of $\hat{f}_{j}$, we know that
\[
\mathcal{E}_{S_{j}}(\hat{f}_{j}) + \lambda\|\hat{f}_{j}\|_{K}^p  \leq \mathcal{E}_{S_{j}}(f_{j}) + \lambda\|f_{j}\|_{K}^p  \leq  \mathcal{E}_{S_{j}}(f) + \lambda\|f\|_{K}^p
\]
Therefore,
\begin{eqnarray} \label{Thm2-2}
B_{j} = \mathcal{E}_{S_{j}}(\hat{f}_{j}) -\mathcal{E}_{S_{j}}(f)) \leq \lambda\|f\|_{K}^p -  \lambda\|\hat{f}_{j}\|_{K}^p  \leq \lambda\|f\|_{K}^p.
\end{eqnarray}
This implies that the first term of (\ref{Thm2-1}) is bounded by $m \lambda\|f\|_{K}^p/N$.

We now turn to bound the second term of (\ref{Thm2-1}). Specifically, we further decompose $U_{j}$ by
\begin{eqnarray*}
U_{j} &=& u_{1j} + u_{2j} + u_{3j} + u_{4j} + B_{j} \\
&\leq& u_{1j} + u_{2j} + u_{3j} + u_{4j} + \lambda\|f\|_{K}^p,
\end{eqnarray*}
where
\begin{eqnarray*}
u_{1j} &=&  \mathcal{E}_{S/S_{j}}(\hat{f}_{j})-\mathcal{E}_{S/S_{j}}(f_{\rho})-\mathcal{E}(\hat{f}_{j})+\mathcal{E}(f_{\rho}) \\
u_{2j} &=&  \mathcal{E}(f)-\mathcal{E}(f_{\rho})-\mathcal{E}_{S/S_{j}}(f)+\mathcal{E}_{S/S_{j}}(f_{\rho})\\
u_{3j} &=&  \mathcal{E}_{S_{j}}(f)-\mathcal{E}_{S_{j}}(f_{\rho})+\mathcal{E}(f_{\rho})-\mathcal{E}(f) \\
u_{4j} &=&  \mathcal{E}(\hat{f}_{j})-\mathcal{E}(f_{\rho})-\mathcal{E}_{S_{j}}(\hat{f}_{j})+\mathcal{E}_{S_{j}}(f_{\rho})
\end{eqnarray*}
Note that $\hat{f}_{j}$ is independent of $S/S_{j}$. Proposition \ref{firstbound} readily implies that, with probability at least $1- \delta$,
\begin{eqnarray*}
u_{1j} &\leq& 4LM \left(\frac{\log(1/\delta)}{N-n}+\sqrt{\frac{\log(1/\delta)}{N-n}}\right), \\
u_{2j} &\leq& 2L \|f - f_{\rho}\|_{\infty} \left(\frac{\log(1/\delta)}{N-n}+\sqrt{\frac{\log(1/\delta)}{N-n}}\right), \\
u_{3j} &\leq& 2L \|f - f_{\rho}\|_{\infty} \left(\frac{\log(1/\delta)}{n}+\sqrt{\frac{\log(1/\delta)}{n}}\right).
\end{eqnarray*}
Also, by applying Proposition \ref{secondbound} with $m=1$, we have, with probability at least $1- \delta$,
\[
u_{4j} \leq  12ML\left(\frac{V(n, \delta) +  \sqrt{V(n, \delta)n}}{n} \right) + n^{- 1/(s+2)},
\]
with the same $V$ defined in Proposition \ref{secondbound}. Consequently, we have, with probability at least $1- \delta$,
\begin{eqnarray}\nonumber
 U_{j} \leq 6LM' \left( \frac{V(n, \delta/4) + \log(4/\delta)}{n} + \frac{\sqrt{\log 4/\delta} + \sqrt{V(n, \delta/4)} }{n^{\frac{1}{2}}} \right) + \frac{1}{n^\frac{1}{2+s}} +  \lambda\|f\|_{K}^p, \\ \label{Thm2-3}
\end{eqnarray}
where $M' = \max\{2M,  \|f - f_{\rho}\|_{\infty}\}$.

Inequalities (\ref{Thm2-2}) and (\ref{Thm2-3}) further imply that, with probability at least $1- \delta$
\begin{eqnarray}\nonumber
\mE_{S}(\bar{f}) - \mE_{S}(f) \leq 6LM' \left( \frac{T_{1}(n, \delta/2)}{n} + \frac{T_{2}(n, \delta/2) }{n^{\frac{1}{2}}} \right) + \frac{1}{n^\frac{1}{2+s}} +  2\lambda\|f\|_{K}^p.
\end{eqnarray}
The theorem is therefore proved.
\end{proof}
Theorem \ref{HypothesisError_Theory} implies that, with appropriate $f$ and $\lambda$, the hypothesis error of DKR has an $O(n^{-1/(2+s)})$ bound in probability. This results is applicable to a general $\bar{f}$ with $m\geq1$, which incorporates the diverging $m$ situations.

\subsection{Generalization Bound of DKR} \label{Gbound}

With the aid of Theorems 1-2, we obtain a probability bound for the generalization error of $\bar{f}$ as the following theorem.

\begin{theorem}(Generalization Error)\label{maintheorem}
Suppose that Conditions A1-A3 are satisfied.  When $N$ is sufficiently large, for any $0<\delta<1$,
\begin{eqnarray}\nonumber
\mE(\bar{f})-\mE(f_{\rho}) \leq 24LM \left( \frac{T_{1}(n, \delta/4)}{n} + \frac{T_{2}(n, \delta/4) }{n^{\frac{1}{2}}} \right) + \frac{2 + L}{n^\frac{1}{2+s}} +  2\lambda\|f_{0}\|_{K}^p
\end{eqnarray}
with probability at least $1- \delta$, where $f_{0} \in \mathcal{H}_{\mathcal{K}}$ and $\|f_{0} - f_{\rho}\|_{\infty}\leq N^{-1}$.
\end{theorem}
\begin{proof}
Under Conditions A1 and A3, for any $N\geq1$, there exists a $f_{0} \in \mathcal{H}_{\mathcal{K}}$ such that $\|f_{0} - f_{\rho}\|<N^{-1}$. Under A2, this also implies that (\ref{approximationerror}) is bounded by $L/N \leq L/n^{1/(2+s)}$.
Clearly, when $N$ is sufficiently large, $M'= \max(2M, \|f_{0} - f_{\rho}\| ) = 2M$.
The theorem is a direct result by applying Theorems \ref{upperboundsampelerror}-\ref{HypothesisError_Theory} to (\ref{sampleerror}) and (\ref{hypothesiserror}) with $f=f_{0}$.
\end{proof}

Theorem \ref{maintheorem} suggests that, if we set $\lambda = o(\|f_{0}\|_{K}^{-p} n^{-1/(2+s)})$,
the generalization error of $\bar{f}$ is bounded by an $O(n^{-1/(2+s)})$ term in probability.
In other words, as $n\rightarrow\infty$, a properly tuned DKR leads to an estimator that achieves the oracle predictive power.
This justifies the feasibility of using divide-and conquer strategy for the kernel-based regression analysis.
Under the assumption that $f_{\rho} \in \mathcal{H}_{\mathcal{K}}$, we have $f_{0}=f_{\rho}$ and thus $\bar{f}$ is feasible with $\lambda=o(n^{-1/(2+s)})$.
Moreover, when DKR is conducted with Gaussian kernels, Condition A3 is satisfied with any $s>0$ and thus $\mE(\bar{f})$ enjoys a nearly $O_{p}(n^{-1/2})$ convergence rate to $\mE(f_{\rho})$.

Theorem \ref{maintheorem} provides theoretical support for the distributed learning framework (Algorithm 1).
It also reveals that the convergence rate of $\mE(\bar{f})$ is related to the scale of local sample size $n$.
This seems to be reasonable, because $\hat{f}_{j}$ is biased from $f_{\rho}$ under a general setup.
The individual bias of $\hat{f}_{j}$ may diminish as $n$ increase. It, however, would not be balanced off by taking the average of $\hat{f}_{j}$s for $j=1, \ldots, m$.
As a result, the generalization bound of $\bar{f}$ is determined by the largest bias among the $m$ $\hat{f}_{j}$s.
When $\hat{f}_{j}$ is (nearly) unbiased, its generalization performance is mainly affected by its variance. In that case, $\bar{f}$ is likely to achieve a faster convergence rate by averaging over $\hat{f}_{j}$s. We use the following corollary to show some insights on this point.

\begin{corollary} \label{cor1}
Suppose that DKR is conducted with the quadratic loss and $\lambda=0$. If $\mathbb{E}[\hat{f}_{j}(x) - f_{\rho}(x)]=0$ for any $x \in \mathcal{X}$, then under Conditions A1-A3, we have
\[
\mE(\bar{f})-\mE(f_{\rho}) = O_{p}\left(\frac{1}{mn^{\frac{1}{2+s}}}\right).
\]
\end{corollary}
\begin{proof}
Let $\rho_{X}$ be the marginal distribution of $X$. When the quadratic loss is used, we have
\begin{eqnarray}\label{cor1-1}
\mE(\bar{f})-\mE(f_{\rho}) = \|\bar{f} -  f_{\rho}\|^{2}_{\rho_{X}} = \int_{\mathcal{X}}  (\bar{f}(X) -  f_{\rho}(X))^{2}    d \rho_{X}
\end{eqnarray}
Since we assume $\mathbb{E}[\hat{f}_{j}(x)] = f_{\rho}(x)$ for any $x\in \mathcal{X}$, (\ref{cor1-1}) implies that
\begin{eqnarray}\nonumber
\mathbb{E}[\mE(\bar{f})-\mE(f_{\rho})] &=&  \int_{S} \int_{\mathcal{X}} (\bar{f}(X) -  f_{\rho}(X))^{2}    d \rho_{X} d\rho \\ \nonumber
&=&  \int_{\mathcal{X}} (\mathbb{E}[\bar{f}(X)- f_{\rho}(X)])^{2}   d \rho_{X} +  \int_{\mathcal{X}} \mathbb{E}[\bar{f}(X)- f_{\rho}(X)]^2   d \rho_{X} \\ \nonumber
&=&  \frac{1}{m} \int_{\mathcal{X}} \mathbb{E} [\hat{f}_{1}(X) - f_{\rho}(X)]^{2}  d \rho_{X} \\ \label{cor1-2}
&=&  \frac{1}{m} \mathbb{E}[ \mE(\hat{f}_{1})-\mE(f_{\rho})  ].
\end{eqnarray}
Applying Theorem \ref{maintheorem} with $m=1$ and $\lambda=0$, we have, for some generic constant $C>0$,
\begin{eqnarray} \label{cor1-3}
\mathbb{P}\left\{\mE(\hat{f}_{1})-\mE(f_{\rho}) > C \log(8/\delta) {n^{-\frac{1}{2+s}}} \right\} \leq \delta
\end{eqnarray}
Let $t=  C \log(8/\delta) {n^{-\frac{1}{2+s}}}$. Inequality (\ref{cor1-3}) implies that
\begin{eqnarray*}\nonumber
\mathbb{E}[\mE(\hat{f}_{1}) - \mE(f_{\rho})] &=& \int_{0}^{\infty} \mathbb{P}\left\{\mE(\hat{f}_{1})-\mE(f_{\rho}) > t\right\} dt \\
&\leq& \int_{0}^{\infty} 8 \exp\left\{- C^{-1}n^{\frac{1}{2+s}}t \right\} dt \\
&\leq& 8 C n^{-\frac{1}{2+s}}.
\end{eqnarray*}
This together with (\ref{cor1-2}) implies that $\mathbb{E}[\mE(\bar{f})-\mE(f_{\rho})] = O(m^{-1}n^{-\frac{1}{2+s}})$, which further implies the corollary.
\end{proof}

Corollary \ref{cor1} is only conceptual, because it is usually difficult to construct an unbiased $\hat{f}_{j}$ without strong prior knowledge.
Nevertheless, it sheds light on designing more efficient DKR with less biased sub-estimators.
In practice, this may be conducted by choosing a small $\lambda$ or using some debiasing techniques in Algorithm 1.
In this paper, we focus on providing a general feasibility support for DKR and leave this issue for the future research.

It should also be noted that, under Theorem \ref{maintheorem}, DKR is feasible only when $n \rightarrow \infty$ or equivalently $m =o(N)$.
This means that, to have DKR work well, the sample size in each local machine should be large enough. This seems to be a natural condition, because for a large-$m$-small-$n$ situation, each local output $\hat{f}_{j}$ is unlikely to provide a meaningful estimate.
As a consequence, the global estimation $\bar{f}_{\lambda }$ may not be well constructed neither. In real applications, an appropriate $m$ should be used such that the associated DKR achieves a good balance of algorithmic accuracy and computational efficiency.

\section{Numerical Studies}\label{sectoinapplication}

We evaluate the finite sample performance of DKR through both simulation and real
data examples. In particular, we assess the distributive strategy for several popular regression methods in terms of both computational efficiency and generalization capability. All numerical studies are implemented by MATLAB 8.2 on a windows workstation with 8-core 3.07GHz CPUs.

\begin{figure}[t]
  \centering
  \includegraphics[width= 10cm]{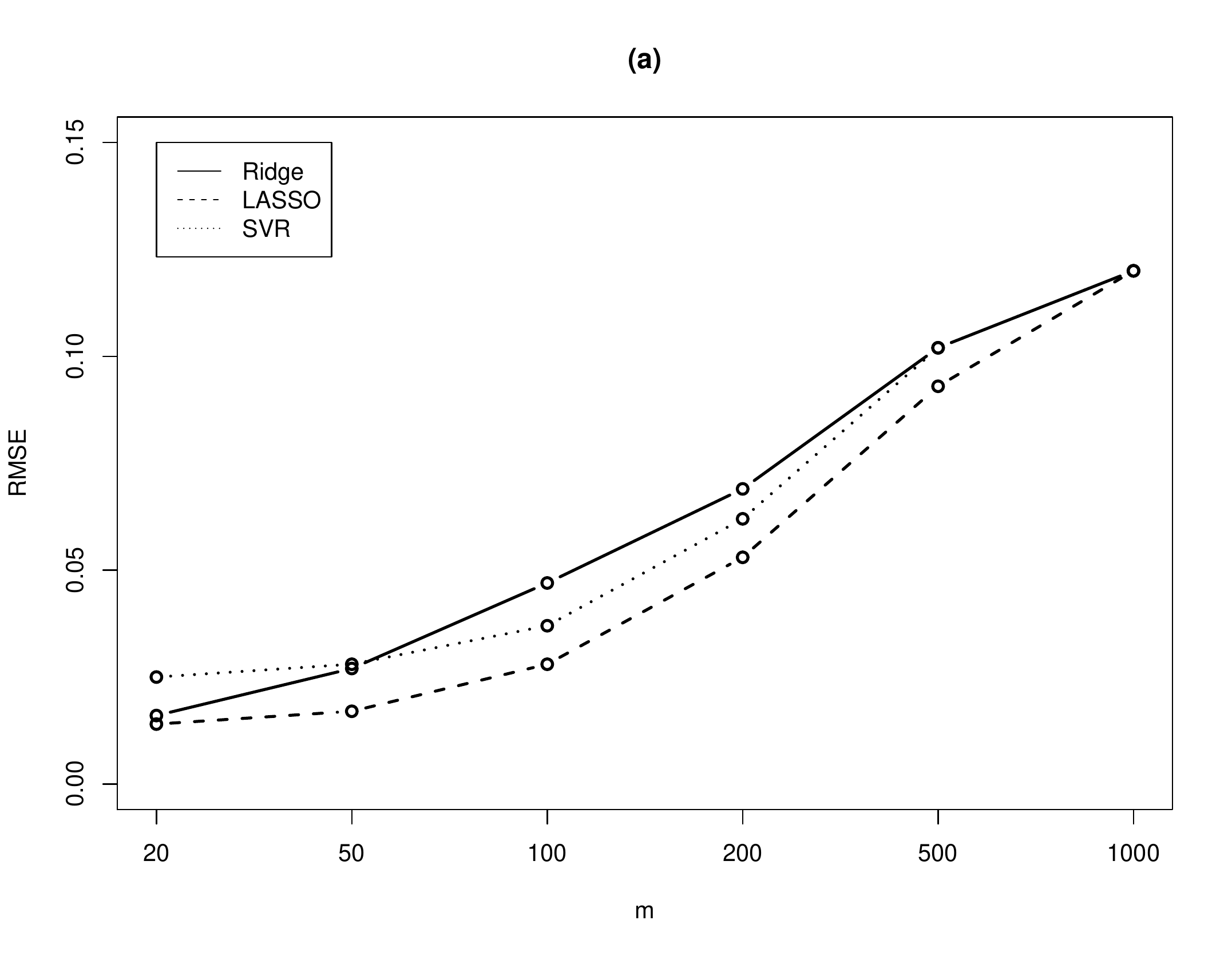} \\
  \includegraphics[width= 10cm]{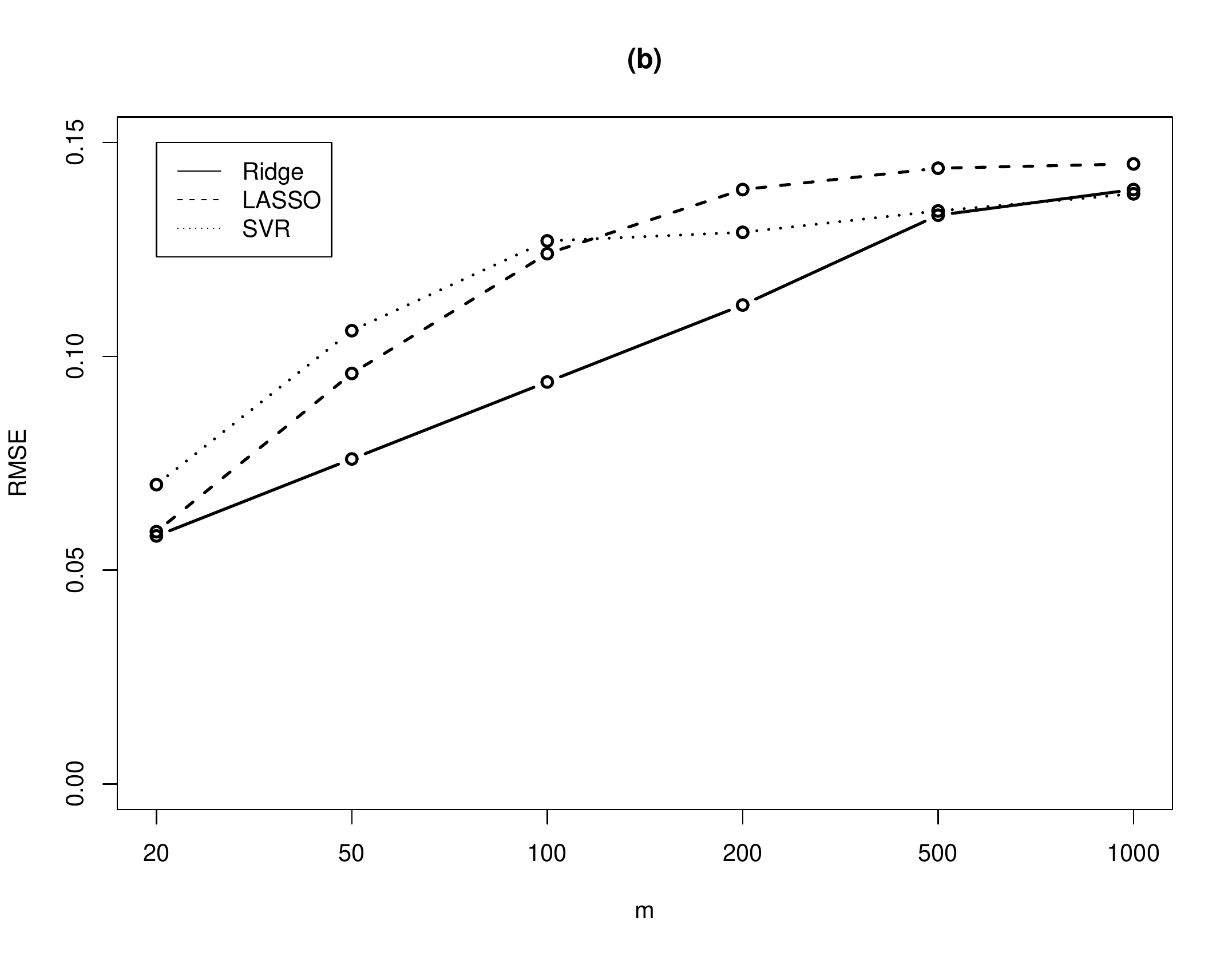}
  \caption{The generalization performance of DKR in Case (i). Plot (a): RMSE for $\bar{f}$; Plot (b): RMSE for $\hat{f}_{1}$.} \label{Fig3}
\end{figure}

\begin{figure}[h]
  \centering
  \includegraphics[width= 10cm]{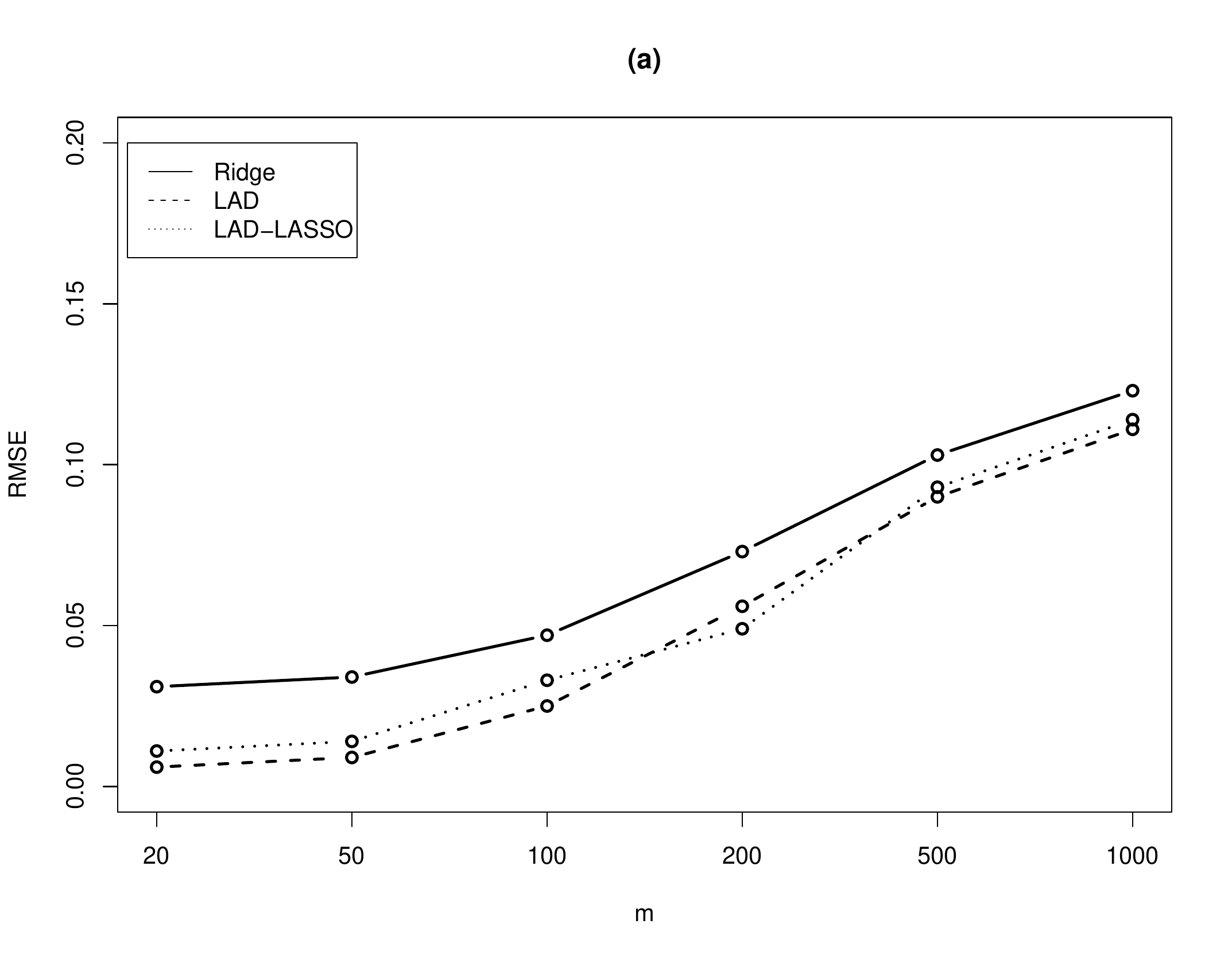} \\
   \includegraphics[width= 10cm]{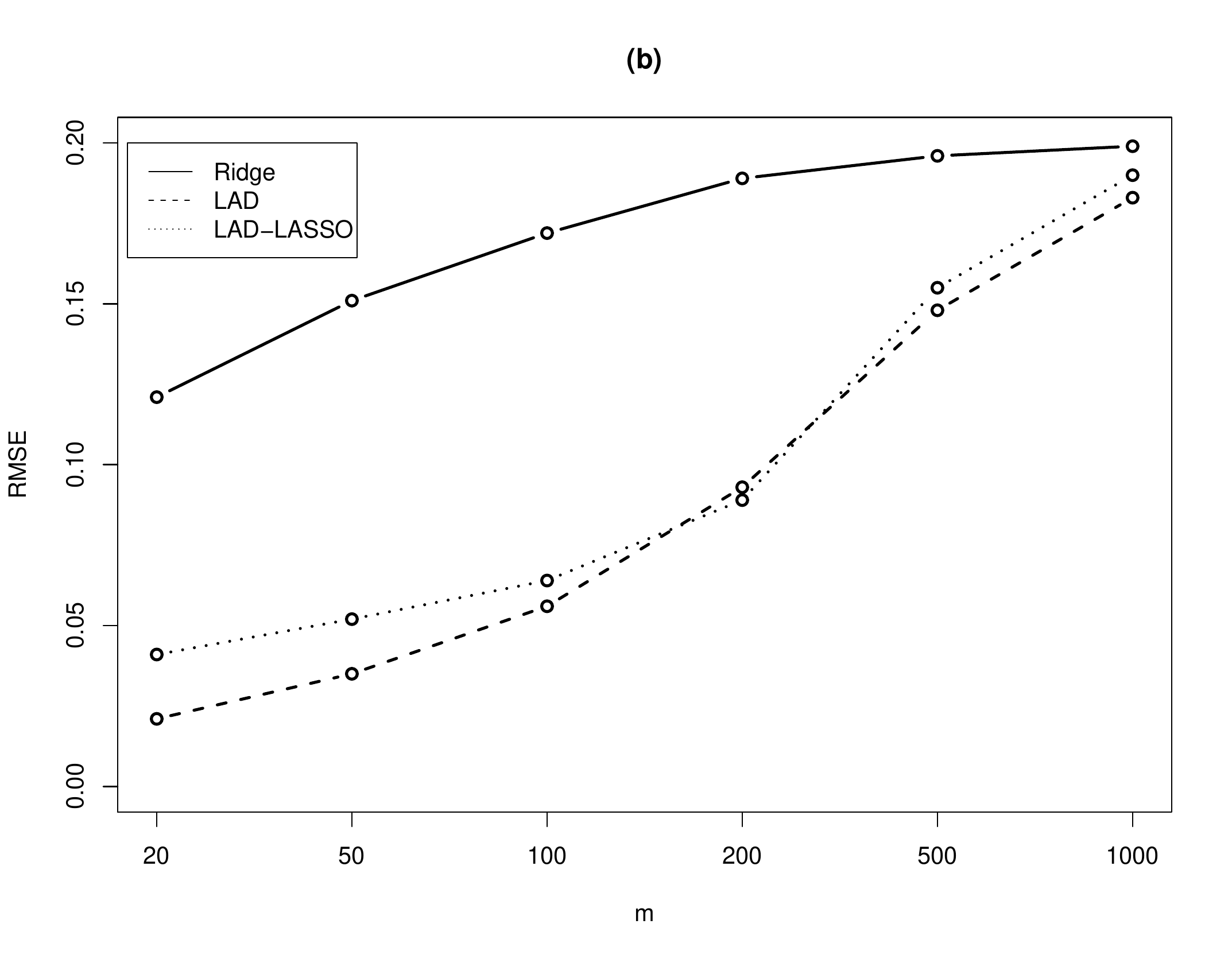}
  \caption{The generalization performance of DKR in Case (ii). Plot (a): RMSE for $\bar{f}$; Plot (b): RMSE for $\hat{f}_{1}$. } \label{Fig4}
\end{figure}

\subsection{Simulation}

In the simulation studies, we assess DKR on a hypothetical learning problem with $d = 2$.
Specifically, we generate independent observations based on model
\begin{eqnarray} \label{datagen}
Y =  \mbox{sinc}(20X_{1}-10) \times \mbox{sinc}(20X_{2}-10) + \epsilon,
\end{eqnarray}
where $(X_{1}, X_{2})$ denotes the two attributes of covariate $X$, $\epsilon$ is an observational noise, and
\[
\mbox{sinc}(x) =  \left\{
\begin{array}{cc}
\frac{\sin (x)}{x}, & x \neq 0\\
1, & x= 0
\end{array}.
\right.
\]
The values of $(X_1, X_2)$ are sampled based on a uniform distribution on $[0, 1] \times [0, 1]$.

We evaluate DKR based on model (\ref{datagen}) under two cases: (i) we set $N=100,000$ and generate data with $\epsilon \sim N(0, 0.2)$;
(ii) we generate $N_{1}= 80,000$ samples with $\epsilon \sim N(0, 0.1)$ and $N_{2}=20,000$ samples with $\epsilon \sim U[-2, 2]$.
The second case is designed such that the data contain about 20\% outliers. This setup poses further challenges for DKR in learning the relationship between $Y$ and $X$.

Regarding the implementation of DKR, we set the number of partition $m= 20, 50, 100, 200, 500$, and $1000$, so that the minimum sample size in each local machine is 100.
We set the thresholding value $M=1$ and build the dictionary $\mathcal{H}_{K}$ by the popular Gaussian kernel
\begin{equation} \label{gaussian}
K(x_{1}, x_{2})= \exp (- \left\Vert x_{1} - x_{2} \right\Vert _{2}^{2} /  \tau^{2} )
\end{equation}
with $\tau=0.05$.
In Case (i), we conduct DKR with three popular regression methods under framework (\ref{RERM}): ridge regression ($L_{2}$-loss plus $L_{2}$-regularization), LASSO ($L_{2}$-loss plus $L_{1}$-regularization), and SVR ($\varepsilon$-intensive-loss plus $L_{2}$-regularization); in Case (ii), we conduct DKR based on two robust regression methods: LAD ($L_{1}$-loss plus $L_{2}$-regularization) and LAD-LASSO ($L_{1}$-loss plus $L_{1}$-regularization).
In our simulations, we choose the tuning parameter $\lambda$ based on a few pilot runs of DKR with $m=20$ and use the standard MATLAB packages for computing the corresponding regression estimators.

To assess the generalization capability of DKR, we generate an independent testing set $\{(\tilde{y}_{i}, \tilde{x}_{i}), \ i=1, \ldots, n_{t}\}$ of size $n_t = 5000$ from model (\ref{datagen}) with $\epsilon =0$ and compute
\[
\mbox{RMSE}(\bar{f}) =  \left\{ \frac{1}{n_{t}} \sum \limits_{i=1}^{n_{t}} \left\vert \bar{f}(\tilde{x}_{i}) - \tilde{y}_{i} \right\vert ^{2} \right\}^{1/2}.
\]
We report the averaged RMSE of DKR for each setup based on 100 repetitions. For comparison, we also report the RMSE of the corresponding standard (non-distributive) regression method based on $1/m$ of the data.

\begin{table}[t]
\centering
\caption{Simulation results: averaged computational time of DKR in second.} \label{Tab1} \vspace{0.2cm}
\begin{tabular}{cccccccc}
  \hline \hline
        & $m=$       & 20 & 50 & 100 & 200 & 500 & 1000   \\ \hline
Case (i)  & Ridge      & 27.6 & 1.91 & 0.26 & 0.04 & $<0.01$ & $< 0.01$ \\
        & LASSO      & 74.8  & 13.6 & 4.93 & 2.54 & 1.91 & 1.20  \\ \vspace{0.1cm}
        & SVR        & 0.94 & 0.25 & 0.09 & 0.04 & 0.02 & 0.01   \\
  \hline

Case (ii)  & Ridge      & 28.1 & 1.94 & 0.26 & 0.04 & $< 0.01$ & $< 0.01$ \\
        & LAD        & 112 & 15.2 & 2.67 & 0.76 & 0.23 & 0.16  \\
        & LAD-LASSO  & 104 & 17.4 & 2.31 & 0.59 & 0.17 & 0.08   \\
  \hline \hline
\end{tabular}

\end{table}

The simulation results are shown in Figures \ref{Fig3}-\ref{Fig4}, where the associated computational cost is given in Table \ref{Tab1}.
We observe that, when $m$ is moderate, the DKR approach performs quite well in achieving a low RMSE for all tested regression methods.
This partially justifies the feasibility result obtained in this work. In our setup, choosing $m \in (50, 100)$ seems to be the most attractive,
because the associated DKR estimator enjoys a strong generalization capability at a low computational cost.
Clearly, by using multiple machines, DKR retains almost the same computational cost as the standard non-distributive method using only $1/m$ of the data.
Meanwhile, with a moderate $m$, it significantly improves the resulting estimator over the single machine-based local output.
The framework of DKR therefore serves as a viable route for conducting efficient leaning for big data.

It should also be noted that the performance of DKR may deteriorate when $m$ is overly large. In Case (i) with $m=1000$, DKR does not help much in reducing the RMSE of the single-machine-based estimator. As discussed in Section \ref{Gbound}, this might be caused by the estimation bias and insufficient sample size for each local machine. In principle, a smaller $m$ helps to improve the effectiveness of DKR, but it also leads to a higher computational cost. In practice, one should conduct DKR with different choices of $m$ and select an appropriate value based on specific situations. It might be a good idea to set $m$ as the smallest value within the affordable computational budget.

DKR also inherits reasonable robustness against outliers from the associated local outputs. This is revealed by the low RMSE of $\bar{f}$ conducted on LAD and LAD-LASSO in Case (ii) with $m \leq 50$.

\subsection{Real data example}

We apply DKR to analyze a real world dataset, which contains 583,250 instances of Twitter discussions on topics related to new technology in 2013.
Each instance is described by $d=77$ features related to that discussion. It is of interest to predict the number of active discussions ($Y$) based on these features ($X$).
To facilitate the computing process, we include the instances with $Y \in  [20, 200]$ in our analysis, which leads to a training set with size $174, 507$ and a testing set with size $19, 390$.
We standardize each attribute of $X$ such that it has a zero mean and a unit standard deviation. Readers may refer to \emph{Buzz Data} on \emph{http://archive.ics.uci.edu/ml/datasets.html} for more detailed information about this dataset.

\begin{table}[h]
{\centering
\caption{RMSE for the analysis of Buzz data.} \label{Tab2} \vspace{0.2cm}
\begin{tabular}{cccccc}
  \hline \hline
   $m=$       & 40 & 120 & 300 & 500 & 1000  \\ \hline
   Ridge      &  24.8 & 25.3 & 25.6 & 25.9 & 26.5   \\
   LASSO      &  24.9 & 25.3 & 25.6 & 26.0 & 26.4   \\
   LAD        &  25.1 & 25.4 & 25.9 & 26.0 & 26.3   \\
  \hline \hline
\end{tabular}

}
\end{table}

Similar to our simulation studies, we build $\mathcal{H}_{K}$ based on the Gaussian kernel (\ref{gaussian}) with $\tau=10$.
We set $m= (40,120,300,500,1000)$ and apply DKR to the training sample with Ridge, LASSO, and LAD. We summarize the analysis in term of RMSE based on the testing sample, which is shown in Table \ref{Tab2}. Like many other social media data, this dataset is known to be noisy and highly skewed. Thus, the results in Table \ref{Tab2} indicate the decent performance of DKR. In this example, we observe that the results are not very sensible to the choice of $m$. Thus, researchers may prefer a larger $m$ for the computational convenience.

\section{Conclusion}\label{sectionConclusion}

In this paper, we studied the distributed kernel regression for learning with big data.
DKR follows from a divide-and-conquer framework, which enables distributive storage and parallel computing.
In DKR, the performance of the global estimator is determined by a uniform bound over the distributed local estimates.
Under mild conditions, we show that DKR provides a consistent estimate that leads to the oracle generalization risk.
Our results offer a general theoretical support for DKR, which is applicable to a broad range of regression methods.
As the first step, the current work focus only on the feasibility of DKR.
It would be important to further investigate its efficiency and develop the corresponding acceleration methods.
Also, it is promising to extend the current distributive framework to other learning tasks, such as classification and variable selection.
We leave all these interesting topics for the future research.

\section*{Acknowledgment}
This work is supported in part by NIDA grants P50 DA10075, P50 DA036107, and the Natural Science Foundation of China grant 11301494.
The authors are grateful to Dr. Xiangyu Chang at Xi'an Jiaotong University (China) and Dr. Jian Fang at Tulane University for their constructive suggestions to this work.


\begin{thebibliography}{9}


\bibitem{Berlinet2004} Berlinet, A. and Thomas-Agnan, C. (2004) \emph{Reproducing Kernel Hilbert Spaces in Probability and Statistics}. Springer, New York.

\bibitem{cao2010approximation} Cao, F., Lin, S. and Xu, Z. (2010) Approximation capability of interpolation neural networks. \emph{Neurocomputing}, \textbf{74} 457-460.



\bibitem{chang2008bigtable} Chang, F.,  Dean, J., Ghemawat, S., Hsieh, W., Wallach, D., Burrows, M., Chandra, T., Fikes, A. and Gruber, R. (2008) Bigtable: A distributed storage system for structured data. \emph{ACM Transactions on Computer Systems}, \textbf{26}, No. 2, Article 4.



\bibitem{chu2006map} Chu, C., Kim, S., Lin, Y., Yu, Y., Bradski, G.,  Ng, A. and Olukotun, K. (2006) Map-reduce for machine learning on multicore. \emph{NIPS}, \textbf{6} 281-288.


\bibitem{christmann2008consistency} Christmann, A. and Steinwart, I. (2008) Consistency of kernel-based quantile regression. \emph{Applied Stochastic Models in Business and Industry}, \textbf{24} 171--183.






\bibitem{dean2008mapreduce} Dean, J. and Ghemawat, S. (2008) MapReduce: simplified data processing on large clusters. \emph{Communications of the ACM}, \textbf{51} 107-113.

\bibitem{fan2008liblinear} Fan, R.,  Chang, K.,  Hsieh, C., Wang, X. and Lin, C. (2008) LIBLINEAR: A library for large linear classification. \emph{The Journal of Machine Learning Research}, \textbf{9} 1871-1874.


\bibitem{ghemawat2003google} Ghemawat, S.,  Gobioff, H. and Leung, S. (2003) The Google file system. \emph{ACM SIGOPS Operating Systems Review}, \textbf{37} 29-43.



\bibitem{kimeldorf1971some} Kimeldorf, G. and Wahba, G. (1971) Some results on Tchebycheffian spline functions. \emph{Journal of Mathematical Analysis and Applications}, \textbf{33} 82-95.

\bibitem{kleiner2012big} Kleiner, A.,  Talwalkar, A., Sarkar, P. and Jordan, M. (2012) The big data bootstrap. \emph{arXiv preprint} arXiv:1206.6415.









\bibitem{li2013statistical} Li, R., Lin, D. and Li, B. (2013) Statistical inference in massive data sets. \emph{Applied Stochastic Models in Business and Industry},
\textbf{29} 399-409.

\bibitem{li2007quantile} Li, Y.,  Liu, Y. and Zhu, J. (2007) Quantile regression in reproducing kernel Hilbert spaces. \emph{Journal of the American Statistical Association}, \textbf{102} 255-268.




\bibitem{mackey2011divide} Mackey, L., Talwalkar, A. and Jordan, M. (2011) Divide-and-conquer matrix factorization. \emph{arXiv preprint}, arXiv:1107.0789.

\bibitem{mcdonald2010distributed} McDonald, R., Hall, K. and Mann, G. (2010) Distributed training strategies for the structured perceptron. \emph{Proceedings of The 2010 Annual Conference of the North American Chapter of the Association for Computational Linguistics}, 456-464. Los Angeles, CA.

\bibitem{Micchelli-et-al(2006)} Micchelli, C., Xu, Y. and Zhang, H. (2006) Universal Kernels. \emph{Jounal of Machine Learning Research}, \textbf{7} 2651-2667.



\bibitem{narcowich2006sobolev} Narcowich, F., Ward, J. and Wendland, H. (2006) Sobolev error estimates and a Bernstein inequality for scattered data interpolation via radial basis functions. \emph{Constructive Approximation}, \textbf{24} 175-186.








\bibitem{scho2001} Sch\"{o}lkopf, B., Herbrich, R. and Smola, A. J. (2001) A Generalized Representer Theorem. \emph{Lecture Notes in Computer Science}, \textbf{2111} 416–426.


\bibitem{scholkopf2001learning} Sch{\"o}lkopf, B. and Smola, A. (2001) Learning with kernels: support vector machines, regularization, optimization, and beyond. The MIT Press. Cambridge, MA.






\bibitem{takeuchi2006nonparametric} Takeuchi, I., Le, Q., Sears, T. and Smola, A. (2006) Nonparametric quantile estimation.  \emph{The Journal of Machine Learning Research}, \textbf{7} 1231-1264.


\bibitem{vapnik2000nature} Vapnik, V. (2000) The nature of statistical learning theory. Springer. NewYork, NY.


\bibitem{wahba1990} Wahba, G. (1990)  Spline models for observational data. SIAM.  Philadelphia, PA.



\bibitem{wright2009robust} Wright, J., Ganesh, A., Rao, S., Peng, Y. and Ma, Y. (2009) Robust principal component analysis: Exact recovery of corrupted low-rank matrices via convex optimization. \emph{Advances in neural information processing systems}, \textbf{22} 2080-2088.



 \bibitem{wu2006learning} Wu, Q., Ying, Y. and Zhou, D. (2006) Learning rates of least-square regularized regression. \emph{Foundations of Computational Mathematics}, \textbf{6} 171-192.

\bibitem{wu2014data} Wu, X., Zhu, X., Wu, G. and Ding, W. (2014) Data mining with big data. \emph{IEEE Transactions on Knowledge and Data Engineering},
\textbf{26} 97-107.

\bibitem{xiang2012approximation} Xiang, D., Hu, T. and Zhou, D. (2012) Approximation analysis of learning algorithms for support vector regression and quantile regression. Journal of Applied Mathematics, \textbf{2012} pp.17.

\bibitem{zhao2012recursive} Zhao, Q., Meng, D. and Xu, Z. (2012) A recursive divide-and-conquer approach for sparse principal component analysis. \emph{arXiv preprint}, arXiv:1211.7219.

\bibitem{zhang2005learning} Zhang, T. (2005) Learning bounds for kernel regression using effective data dimensionality. \emph{Neural Computation}, \textbf{17} 2077-2098.

\bibitem{zhang2013divide} Zhang, Y., Duchi, J. and Wainwright, M. (2013) Divide and Conquer Kernel Ridge Regression: A Distributed Algorithm with Minimax Optimal Rates. \emph{arXiv preprint}, arXiv:1305.5029

\bibitem{zhou2002covering} Zhou, D. (2002) The covering number in learning theory. \emph{Journal of Complexity}, \textbf{18} 739-767.

\bibitem{zhou2003capacity} Zhou, D. (2003) Capacity of reproducing kernel spaces in learning theory. \emph{IEEE Transactions on Information Theory}, \textbf{49} 1743-1752.

\bibitem{zou2013generalization} Zou, B., Li, L., Xu, Z., Luo, T and Tang, Y (2013) Generalization Performance of Fisher Linear Discriminant Based on Markov Sampling. \emph{IEEE Transactions on Neural Networks and Learning Systems}, \textbf{24} 288-300.



\end{thebibliography}

\end{document}